\documentclass{article}
\usepackage{ijcai25}

\usepackage{times}
\usepackage{soul}
\usepackage{url}
\usepackage[hidelinks]{hyperref}
\usepackage[utf8]{inputenc}
\usepackage[small]{caption}
\usepackage{graphicx}
\usepackage{amsmath}
\usepackage{amsthm}
\usepackage{booktabs}
\usepackage{algorithm}
\usepackage{algorithmic}
\usepackage[switch]{lineno}
\usepackage{amssymb}
\usepackage{multirow}
\usepackage{subfigure}
\usepackage{hyperref}

\newtheorem{theorem}{Theorem}

\newtheorem{lemma}{Lemma}

\title{On the Discrimination and Consistency for Exemplar-Free \\ Class Incremental Learning}

\author{
    Tianqi Wang\textsuperscript{1,2}, Jingcai Guo\textsuperscript{1,$\ast$}, Depeng Li\textsuperscript{3}, Zhi Chen\textsuperscript{4}
    \affiliations
    \textsuperscript{1}The Hong Kong Polytechnic University~\textsuperscript{2}University College London~\\
    \textsuperscript{3}Huazhong University of Science and Technology~\textsuperscript{4}The University of Queensland
    \emails
    tianqi.wang.23@ucl.ac.uk,
    jc-jingcai.guo@polyu.edu.hk,
    dpli@hust.edu.cn,
    zhi.chen@uq.edu.au\\
    \textbf{$\ast$Corresponding author: Jingcai Guo}
}

\begin{document}

\maketitle

\begin{abstract}

Exemplar-free class incremental learning (EF-CIL) is a nontrivial task that requires continuously enriching model capability with new classes while maintaining previously learned knowledge without storing and replaying any old class exemplars.
%
An emerging theory-guided framework for CIL trains task-specific models for a shared network, shifting the pressure of forgetting to task-id prediction. 
In EF-CIL, 
task-id prediction is more challenging due to the lack of inter-task interaction (e.g., replays of exemplars). To address this issue, we conduct a theoretical analysis of the importance and feasibility of preserving a discriminative and consistent feature space, upon which we propose a novel method termed DCNet. 
%
Concretely, it progressively maps class representations into a hyperspherical space, in which different classes are orthogonally distributed to achieve ample inter-class separation. Meanwhile, it also introduces compensatory training to adaptively adjust supervision intensity, thereby aligning the degree of intra-class aggregation. 
Extensive experiments and theoretical analysis verified the superiority of the proposed DCNet\footnote{Code: \url{https://anonymous.4open.science/r/DCNet-70E9}.}.
 
\end{abstract}

\section{Introduction}\label{introduction}
Deep neural networks have achieved state-of-the-art performance in various tasks, yet they often struggle with Class Incremental Learning (CIL). In CIL, the model is constrained to learn new classes on non-stationary data distributions. This scenario can result in Catastrophic Forgetting (CF) ~\cite{mccloskey1989catastrophic}, as new parameters overwrite those learned for previous tasks. Meanwhile, CIL focuses on not relying on privileged information such as task-ids during inference~\cite{wang2023beef}.

Exemplar-based methods, which preserve a portion of the samples from previous tasks for replay, have demonstrated strong performance in CIL. However, in the era of connectivity, data privacy has become increasingly crucial. The rising concern over data privacy conflicts with the exemplar-based approach~\cite{zhuang2022acil}, thereby constraining its applicability. Recently, EF-CIL has attracted considerable attention because it entirely eliminates the need for replay samples, making it suitable for deployment in scenarios where privacy preservation and storage limitation are critical. Despite this advantage, existing EF-CIL methods are prone to more severe CF because training on a new task overwrites the parameter space learned for previous tasks. Classic strategies reduce the alteration of important weight by imposing constraints such as regularization using the Fisher information matrix~\cite{PNAS2017EWC}. More recent method leverages the Hilbert-Schmidt independence criterion for more stringent constraints~\cite{li2024CLDNet}. Rather than directly focusing on the weight, alternative approaches aim to maintain the semantic consistency of the prior feature space, typically through the use of class prototypes~\cite{zhu2021prototype,toldo2022Fusion,magistri2024EFC}. Additionally, there are strategies that utilize the generation of pseudo-samples to recover old prototypes~\cite{meng2025Diffclass}. 

Parallel to these approaches, theoretical research suggests that a proficient CIL model can be broken down into a task-incremental learning (TIL) + out-of-distribution (OOD) task~\cite{kim2022theoretical}. TIL typically involves training a separate model for each individual task and selecting appropriate inference paths or output heads through known task-ids. Based on this, the TIL+OOD architecture entails training a TIL-like model that constructs a new OOD classifier when faced with a new task. Consequently, an independent OOD classifier for each task would concurrently handle in-distribution (IND) classification, which is \emph{within-task prediction}, and OOD detection to ensure accurate \emph{task-id prediction}. During inference, for each test sample, the framework evaluates the probabilities of both to make a decision. This architecture facilitates the sharing of inter-task generalized knowledge while preserving intra-task specific knowledge, thereby demonstrating superior performance on CIL tasks. However, the TIL+OOD architecture results in task isolation during training, as it prevents access to the embedded representations and output magnitudes of previous tasks while learning the current task. Paradoxically, effective decision-making requires comparing outputs across the incremental sequence. This contradiction ultimately leads to a performance bottleneck in task-id prediction that requires inter-task interaction. Previous researches~\cite{kim2022MORE,kim2023ROW,lin2024TPL} have primarily focused on using replay samples to facilitate interaction, but this invades data privacy. Our work emphasizes the implementation of inter-task interactions for the TIL+OOD framework in the context of EF-CIL.

From this perspective, we argue that preserving the discriminative and consistent feature space is crucial for enabling effective task interaction. \emph{(i) Enhancing discriminability for a single task}. The feature space generated by general embedding methods, although effective for within-task prediction, frequently falls short in supporting OOD detection~\cite{deng2024EHS,ming2023cider}. This limitation arises because OOD detection necessitates more discriminative features. \emph{(ii) Ensuring consistency across incremental tasks}. Even with perfect OOD detection for each task, the isolation between tasks can lead to varying output magnitudes~\cite{kim2022MORE,kim2022theoretical}. Our theoretical analysis indicates that maintaining discriminative and consistent feature space can be achieved by enhancing inter-class separation and aligning intra-class aggregation.


In this paper, we introduce a novel method to EF-CIL named Discriminative and Consistent Network (DCNet). This multi-head model leverages HAT~\cite{serra2018HAT} to learn task-specific masks for protecting the knowledge and makes decisions by comparing a sequence of OOD classifier outputs. To fully exploit the information in incremental tasks for interaction, DCNet comprises two key components: (1) Incremental Orthogonal Embedding (IOE), where we sequentially generate basis vectors that are orthogonally distributed on the unit hypersphere. The model then embeds the corresponding category features as closely as possible to these predefined vectors. This guarantees that the embedding of each category remains orthogonal to those of prior and future categories, thereby enhancing and aligning intra-task separation. (2) Dynamic Aggregation Compensation (DAC), which addresses the issue due to decreasing model plasticity by adaptively compensating for the reduced feature aggregation of subsequent tasks. DAC brings incremental feature embedding with more even intra-class aggregation. Benefiting from the synergy of these two components, DCNet effectively preserves the discriminative and consistent characteristics of the features, and does not rely on replaying samples or pre-trained models. Our main contributions are threefold:

\begin{itemize}
	\item Theoretical analyses are provided to demonstrate the feasibility of leveraging information in incremental sequences to preserve discriminative and consistent features for the framework of TIL+ODD. To the best of our knowledge, we are the first to formally discuss how to optimize TIL+OOD model in the context of EF-CIL.

	\item DCNet not only incrementally embeds category features on the unit hypersphere but also maintains inter-class orthogonality. Furthermore, additional adaptive compensation helps balance the degree of intra-class aggregation across all tasks.
    
	\item Experiments conducted across multiple benchmark datasets consistently demonstrate that our method achieves highly competitive EF-CIL performance, with an average improvement of 8.33\% over the latest state-of-the-art method on ImageNet-Subset task.
\end{itemize}

\section{Related Work}
\textbf{Class-Incremental Learning.}
CIL necessitates that the model incrementally learn new classes without forgetting
previously acquired knowledge. Classical CIL methods often maintain a certain number of exemplars from previous classes, which are replayed upon the arrival of a new task~\cite{rebuffi2017icarl,buzzega2020DER++,yan2021dynamically,wang2022foster,wang2023beef}. Replay strategy effectively reduces CF; however, concerns over privacy and memory limitation restrict its practicality. Exemplar-free approaches have focused on mitigating CF without relying on replay samples~\cite{zhu2023SSRE,rypesc2023SEED,gomez2025LDC}. EWC~\cite{PNAS2017EWC} employs the Fisher information matrix to constrain significant alterations in the weight space. LwF~\cite{li2017learning} ensures that the output of the current model remains close to that of the previous model. PASS~\cite{zhu2021prototype} leverages self-supervised learning to train a backbone network and maintains the consistency of class prototypes. FeTrIL~\cite{petit2023fetril} transforms old prototype features based on the differences between old and new prototypes. ADC~\cite{goswami2024ADC} uses adversarial samples against old task categories to estimate feature drift. EFC~\cite{magistri2024EFC} identifies critical directions in the feature space for the previous task. However, some EF-CIL methods depend on a large initial task to train the backbone network and subsequently freeze it during increments. More recently, some studies have also explored utilizing pre-trained diffusion model or saliency detection network to mitigate forgetting~\cite{meng2025Diffclass,liu2024TASS}. Our end-to-end approach explores the application of TIL+OOD framework to EF-CIL without depending on a large initial task or a pre-trained model.

\textbf{Task-id Predictor.}
One special approach to addressing the CIL problem involves utilizing multi-head models with task-id prediction. Specifically, CCG~\cite{abati2020CCG} builds a separate network to predict task-id, while iTAML~\cite{rajasegaran2020itaml} necessitates batched samples for task-id prediction during inference. HyperNet~\cite{von2019HyperNet} and PR-Ent~\cite{henning2021PR-Ent} employ entropy for task-id prediction. Prior research has highlighted that the performance bottleneck of these systems stems from failing to realize the relationship between task-id prediction and OOD detection~\cite{kim2022theoretical}. OOD detection requires models not only to accurately identify data from known distributions (i.e., categories learned during training), but also to detect samples outside these distributions (i.e., unknown categories)~\cite{morteza2022provable,ming2023cider,Lu2024PALM}.  Kim et al.~\shortcite{kim2022theoretical} conducted a theoretical analysis of the TIL+OOD architecture, demonstrating its applicability to CIL tasks. Building on this, MORE~\cite{kim2022MORE} and ROW~\cite{kim2023ROW} adopt the same structure using pre-trained models and replay samples. The latest work TPL~\cite{lin2024TPL}, leverages replay samples to construct likelihood ratios, thereby enhancing task-id prediction. Our proposed method also falls under the TIL+OOD framework. However, unlike previous methods, DCNet focuses on efficiently utilizing information from incremental sequences to accomplish inter-task interaction without replay samples.

\section{Theoretical Analysis}\label{Analysis}
\subsection{Class-Incremental Learning Setup}
Class-Incremental Learning (CIL) aims to address a sequence of tasks \(1, \dots, T\). Each task \(t\) consists of an input space \(\mathcal{X}^{(t)}\), a label space \(\mathcal{Y}^{(t)}\), and a training set \(\mathcal{D}^{(t)} = \{(x_j^{(t)}, y_j^{(t)})\}_{j=1}^{N^{(t)}}\), where $N^{(t)}$ is the number of samples. The label spaces of different tasks have no overlap, i.e., \(\mathcal{Y}^{(i)} \cap \mathcal{Y}^{(k)} = \emptyset, \forall i \neq k\). The objective of CIL is to train a progressively updated model that can effectively map the entire input space \(\bigcup_{t=1}^T \mathcal{X}^{(t)}\) to the corresponding label space \(\bigcup_{t=1}^T \mathcal{Y}^{(t)}\). Kim et al.~\shortcite{kim2022theoretical} proposed a novel theory for solving the CIL problem. They decomposed the probability of a sample \(x\) belonging to class \(y_j^{(t)}\) of task \(t\) as follows:
\begin{equation}
P(y_j^{(t)} \mid x) = P(y_j^{(t)} \mid x, t) \cdot P(t \mid x).
\end{equation}

The formulation can be decoupled into two components: \textit{within-task prediction} and \textit{task-id prediction}. In CIL, data from different tasks can be considered as OOD samples to each other. However, only relying on traditional OOD detection methods to build separate model results in task isolation. We highlight that interaction between tasks can be achieved by preserving the discriminative and consistent feature space. In the subsequent sections, we analyze how these properties can be implemented theoretically. Subsection~\ref{OOD Detection Capabilities} underscores the importance of inter-class separation and intra-class aggregation through a theorem. Subsection~\ref{Task Information Interaction} elaborates on how both concepts facilitate information interaction.

\subsection{OOD Detection Capabilities}\label{OOD Detection Capabilities}
We now discuss how inter-class separation and intra-class aggregation affect the performance of OOD detection in the context of EF-CIL. The theory proposed by Morteza and Li~\shortcite{morteza2022provable} demonstrates that the effectiveness of pure OOD detection is closely tied to the distance between the IND and OOD data. We generalize this theory to incremental learning tasks.

Let $\{\mu_{\mathrm{in},i}\}_{i=1}^k$ be the mean vectors of $k$ Gaussian components representing the IND (a task contain $k$ distinct categories). Consider a sequence of OOD Gaussian mean vectors $\{\mu_{\mathrm{out},t}\}_{t=1}^T$, describing a uniformly weighted Gaussian mixture model for OOD data (an incremental sequence contain $T$ tasks), and let $\Sigma$ be the positive definite shared covariance matrix. IND and OOD exhibit distinct and significant differences. Define a scoring function $ES(x)$ that is proportional to the data density, and a measure $D$ of the expectation difference in $ES(x)$ between IND and OOD samples:
\begin{equation}\label{ES(x)}
  ES(x) \;=\; \sum_{i=1}^k \exp~\!\Bigl (-\tfrac{1}{2}\,(x - \mu_i)^\top\,\Sigma^{-1}\,(x - \mu_i)\Bigr),
\end{equation}
\begin{equation}\label{D}
  D \;=\; \mathbb{E}_{x \sim P_\chi^{\mathrm{in}}}\bigl[ES(x)\bigr] 
  \;-\; 
  \mathbb{E}_{x \sim P_\chi^{\mathrm{out}}}\bigl[ES(x)\bigr].
\end{equation}

Recall the Mahalanobis distance:
\(d_M(u, v) = 
\sqrt{\,(u - v)^\top\,\Sigma^{-1}\,(u - v)\,}\). Our objective is to investigate the factors influencing $D$, a metric that quantifies the performance of OOD detection in incremental sequences. We investigate the upper bound of $D$, which is derived using the Total Variation and the Pinsker's inequality (see Lemma~\ref{Total Variation} and ~\ref{Pinsker's inequality}).

\begin{lemma} \label{upper bound}
Let
\(\alpha_{i,t} := 
\tfrac{1}{2}\, d_M\!\bigl(\mu_{\mathrm{in},i}, \mu_{\mathrm{out},t}\bigr),\; i = 1,\ldots, k,\; t = 1,\ldots, T\), then we have the following estimate:
\begin{equation} \label{upper bound equation}
\mathbb{E}_{x \sim P_\chi^{\text{in}}}(ES(x)) - \mathbb{E}_{x \sim P_\chi^{\text{out}}}(ES(x)) \leq \frac{1}{T} \sum_{t=1}^T \sum_{i=1}^k \alpha_{i,t}.
\end{equation}

Details of the proof are provided in Appendix~\ref{proof2}.
\end{lemma}

Lemma \ref{upper bound} demonstrates that as $\mu_{\mathrm{in},i}$ and $\mu_{\mathrm{out},t}$ become more distant, the overall OOD detection performance improves. However, within the context of EF-CIL, the OOD data is unavailable, making it challenging to explicitly increase the divergence. Consequently, the subsequent derivation aims to further relax the upper bound while introducing inter-class separation, which can be estimated in EF-CIL.

For any $\mu_{\mathrm{out},t}$, we can always find a nearest $\mu_{\mathrm{in}, i_{0}}$ that satisfies the triangle inequality, denoted as:
\begin{equation}\label{inequality}
d_M(\mu_{\mathrm{in},i}, \mu_{\mathrm{out},t}) \leq d_M(\mu_{\mathrm{out},t}, \mu_{\mathrm{in},i_0}) + d_M(\mu_{\mathrm{in},i_0}, \mu_{\mathrm{in},i}).
\end{equation}

We end up with the following theorem.

\begin{theorem}
\label{thm:extended}
Consider a sequential OOD detection task in the context of CIL, we have the following bounds:
\begin{align} \label{maha}
  & \mathbb{E}_{x \sim P_\chi^{\mathrm{in}}}\bigl[ES(x)\bigr]-
  \mathbb{E}_{x \sim P_\chi^{\mathrm{out}}}\bigl[ES(x)\bigr] \notag \\
  & \le
    \frac{k}{2T} \sum_{t=1}^T d_M\!\bigl(\mu_{\mathrm{out},t}, \mu_{\mathrm{in}, i_{0}}\bigr) +
    \frac{1}{2} \sum_{i=1}^k d_M\!\bigl(\mu_{\mathrm{in}, i_{0}}, \mu_{\mathrm{in},i}\bigr).
\end{align}
\end{theorem}

Theorem~\ref{thm:extended} handles incremental task arrivals and introduces inter-class separation $d_M(\mu_{\mathrm{in}, i_{0}}, \mu_{\mathrm{in},i})$. It can be found that the performance of TIL+OOD approaches can be improved by increasing the inter-class difference of IND prototypes. In essence, if the IND data are well-separated, more space exists in the feature space for OOD samples to be embedded.

Finally, we examine how intra-class aggregation influences detection performance through the shared covariance matrix $\Sigma$. If $\Sigma$ becomes ``smaller'' in the positive-definite ordering (i.e., $\Sigma_a \preceq \Sigma_b$ implies $\Sigma_b - \Sigma_a$ is positive semidefinite), then $\Sigma_a^{-1}$ is ``larger'' compared to $\Sigma_b^{-1}$. Consequently, the same Euclidean displacement leads to a larger Mahalanobis distance under a smaller covariance. For any fixed vector $(u - v)$, if $\Sigma_a \preceq \Sigma_b$, then:
\begin{equation}
(u - v)^\top\,\Sigma_a^{-1}\,(u - v)
\;\;\ge\;\;
(u - v)^\top\,\Sigma_b^{-1}\,(u - v).
\end{equation}

Therefore, when the IND data exhibits a higher degree of intra-class aggregation, OOD samples are pushed farther away in the Mahalanobis distance sense.

\paragraph{Interpretation.} Our analysis highlights two crucial factors.
\textbf{Inter-class separation:} A larger separation between IND prototypes offers a better theoretical margin, especially when the OOD dynamics changes and cannot be estimated.
\textbf{Intra-class aggregation:} A smaller covariance matrix $\Sigma$ indicates that the IND data are more tightly clustered around the class prototype, thereby amplifying Mahalanobis distances to potential OOD samples. 

\begin{figure*}[t]
	\begin{center}		\centerline{\includegraphics[width=1.8\columnwidth]{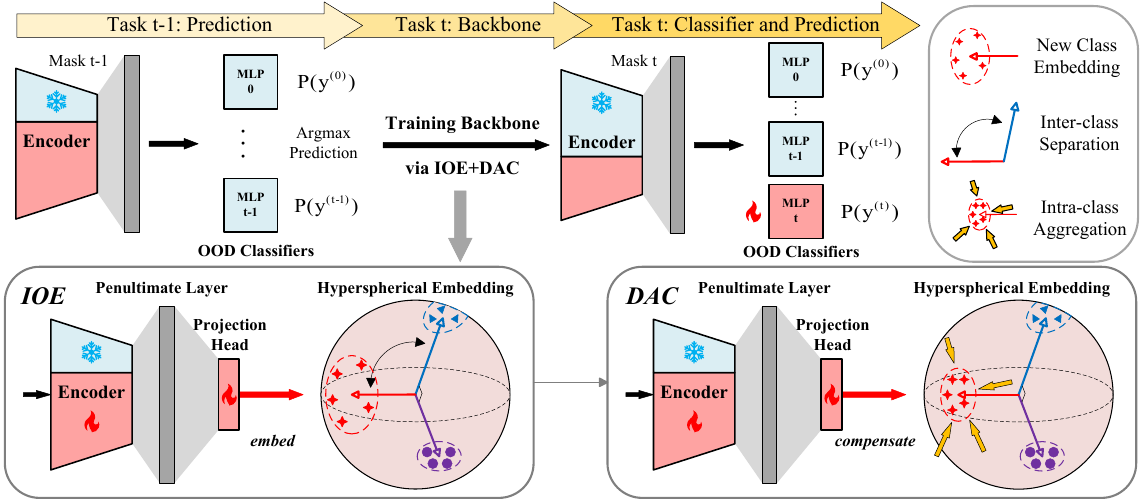}}
        \vskip -0.05in
		\caption{Overview of DCNet. Upon the arrival of task t, DCNet optimizes the learnable part of the backbone and creates a new OOD classifier. Through IOE, DCNet incrementally embeds new features in directions that remain orthogonal to previous categories. Subsequently, the DAC module dynamically compensates for any insufficient aggregation by referencing the degree of aggregation from the previous tasks.}
		\label{Overview_figure}
	\end{center}
	\vskip -0.3in
\end{figure*}

\subsection{Task Information Interaction}\label{Task Information Interaction}
In this subsection, we will discuss the feasibility of employing inter-class separation and intra-class aggregation to preserve discriminative and consistent feature spaces. The TIL+OOD approach consists of a backbone network trained using TIL-like methods and multiple OOD classifiers. During the inference, the final decision $\hat{y}$ is determined by selecting the highest output among the $T$ OOD classifiers:

\begin{equation} \label{decision}
\hat{y} = \arg\max_{1 \leq t \leq T} \oplus P(y^{(t)} \mid x, t),
\end{equation}

\noindent where $\oplus$ denotes the concatenation over the output space. We select the class with the highest softmax probability over each task among all the learned classes. However, even if each individual OOD decision is perfect, Eq. (\ref{decision}) may still lead to incorrect CIL predictions due to varying magnitudes of outputs across different tasks. Breaking the isolation of tasks during the training is essential for achieving comparable outputs. Previous studies~\cite{kim2022MORE,kim2023ROW,lin2024TPL} have facilitated direct information interaction via replay samples; however, these approachs are constrained in the context of EF-CIL. As analyzed in Subsection~\ref{OOD Detection Capabilities}, inter-class separation and intra-class aggregation are critical for maintaining the validity and integrity of the feature space without violating the privacy. Specifically, inter-class separation preserves the discrimination of the feature space, while intra-class aggregation ensures the consistency, thereby enabling effective task information interaction.


\section{Methodology}
\subsection{Overview}\label{overview}
DCNet comprises two essential components: Incremental Orthogonal Embedding (IOE) and Dynamic Aggregation Compensation (DAC), which operate synergistically (see Figure~\ref{Overview_figure}). Consistent with previous works, our approach employs the mask-based method HAT~\cite{serra2018HAT} to shift CF. Specifically, as each task is learned, the model generates a set of masks for important neurons, ensuring that these masks are as compact as possible. Formally, we introduce a loss term $\mathcal{L}_{HAT}$. During the learning of a new task, the masks from the previous model inhibit backpropagation from updating the masked neurons. Since all neurons remain accessible during the forward propagation, inter-task generalized knowledge can be leveraged across all tasks. As the mask is progressively learned, the available weights for updating by subsequent tasks become increasingly sparse, leading to a decrease in model plasticity. This results in the degradation of the feature space for subsequent tasks, where later tasks often exhibit poorer performance compared to earlier tasks in terms of inter-class separation and intra-class aggregation. Through IOE and DAC components, DCNet preserves a discriminative and consistent feature space by facilitating inter-task information interaction.


\subsection{Incremental Orthogonal Embedding (IOE)}\label{IOE}
The core principle of IOE lies in explicitly associating each category with incrementally generated basis vectors that maintain orthogonality, thereby ensuring superior inter-class separation. For an incremental task $t$ with data $\mathbf{x}^{(t)}$, the framework consists of two mappings: An encoder $f: \mathcal{X} \to \mathbb{R}^f$, which maps the input $\mathbf{x}^{(t)}$ to a feature $\mathbf{f}^{(t)} = f(\mathbf{x}^{(t)})$.
A projection head $h: \mathbb{R}^f \to \mathbb{R}^z$, which further maps $\mathbf{f}^{(t)}$ to a embedding $\tilde{\mathbf{z}}^{(t)} = h(\mathbf{f}^{(t)})$. The output embeddings are normalized as $\mathbf{z}^{(t)} = \tilde{\mathbf{z}}^{(t)} / \|\tilde{\mathbf{z}}^{(t)}\|^2$ to reside on a unit hypersphere. 

We aim to more uniformly distribute the unit hypersphere and constrain individual category features within their respective spaces to minimize overlap between categories~\cite{deng2024EHS}. The approach to updating class prototypes is data-driven, which cannot control the location of class prototypes~\cite{ming2023cider}. Based on this, we argue that binding category features to corresponding basis vectors to enforce orthogonality is essential for maintaining distinct boundaries and minimizing interference between different classes. By ensuring $90^\circ$ angular separation between basis vectors, each category can occupy a unique region in the feature space. Furthermore, this component explicitly defines the placement of new class embeddings, thereby preventing potential reductions in inter-class separation that could result from diminished model plasticity.
 
We design an incremental generator that produces predefined, mutually orthogonal basis vectors through a data-independent process. Let $\{\mathbf{\mu}_c^{\mathrm{old}}\}_{c=1}^{C^{\mathrm{old}}}$ be the existing set of $C^{\mathrm{old}}$ basis vectors, each normalized to unit length and mutually orthogonal. When new classes $C^{(t)}$ of task $t$ arrive, each new generated vector $\{\mathbf{\mu}_k^{(t)}\}_{k=1}^{C^{(t)}}$ is also normalized and required to be orthogonal to both existing and newly added vectors. To approximately satisfy these constraints, let $M^* = \{\mathbf{\mu}_c^{\mathrm{old}}, \mathbf{\mu}_k^{(t)}\}$ represent all basis vectors after incrementally adding new classes, and define the following objective:

\begin{equation} \label{ebv}
    M^{*} = \arg\min_{M} \Biggl[
    \sum_{i,j=1}^{C^{(t)}}
    |{\mathbf{\mu}_i^{(t)}}^\top \mathbf{\mu}_j^{(t)}| +
    \sum_{k=1}^{C^{(t)}}\sum_{c=1}^{C^{\mathrm{old}}}
    |{\mathbf{\mu}_k^{(t)}}^\top \mathbf{\mu}_c^{\mathrm{old}}|
\Biggr],
\end{equation}
where each pairwise inner product deviating from zero is penalized. Minimizing these penalty terms encourages all basis vectors to remain mutually orthogonal, thereby providing explicitly defined embedding guidance. 

Having obtained the necessary vectors, we should now concentrate on embedding features in proximity to these basis vectors. Since these basis vectors are distributed on the unit hypersphere, we can model the embedding effectively using the von Mises-Fisher (vMF) distribution~\cite{mardia2009directional}. The vMF distribution serves as a spherical counterpart to Gaussian distributions, designed for unit norm embeddings $\mathbf{z}$ where $\|\mathbf{z}\|^2 = 1$. The probability density function of a unit vector $\mathbf{z} \in \mathbb{R}^d$ belonging to class $k$ in task $t$ is defined as:

\begin{equation} \label{vMF}
p_d(\mathbf{z}^{(t)}; \mu_k^{(t)}, \kappa) = Z_d(\kappa) \exp\left(\kappa  \mathbf{z}^{(t)} {\mu_k^{(t)}}\right),
\end{equation}

\noindent where $\mu_k^{(t)}$ is the generated basis vector for class $k$ in $M^{*}$, $\kappa \geq 0$ represents the concentration parameter controlling the distribution tightness, and $Z_d(\kappa)$ is the normalization factor. The optimized normalized probability of assigning an embedding $\mathbf{z}_{i}^{(t)}$ to category $c_{(i)}$ is given as follows:

\begin{equation} \label{LossIOE}
\mathcal{L}_{\text {IOE}}=-\frac{1}{N^{(t)}} \sum_{i=1}^{N^{(t)}} \log \frac{\exp \left({\mathbf{z}_{i}^{(t)}} {\mu_{c_{(i)}}^{(t)}} / \tau_{\text{IOE}}\right)}{\sum_{j=1}^{C^{(t)}} \exp \left({\mathbf{z}_{i}^{(t)}} {\mu}_{j}^{(t)} / \tau_{\text{IOE}}\right)},
\end{equation}

\noindent  where $c_{(i)}$ denotes the class index of a sample $x_{i}$ in task $t$, $\tau_{\text{IOE}}$ is the fixed temperature. Combining Eqs. (\ref{ebv}) and (\ref{LossIOE}), IOE successfully embeds categories orthogonally on the unit hypersphere. Following previous researches~\cite{zhu2021prototype,Kim2022CLOM,toldo2022bring,magistri2024EFC}, we also employed self-rotation augmentation. It is crucial to highlight that IOE communicates an important information to each task: to precisely delineate inter-class separation and minimize overlap. This information interaction facilitates superior feature discrimination.


\begin{table*}[ht]
\vskip -0.05in
\centering
\begin{center}
\resizebox{0.87\textwidth}{!}{
\begin{tabular}{lcccccccccccc}
\toprule
\multirow{3}{*}{\textbf{Method}} & \multicolumn{4}{c}{\textbf{CIFAR-100}} & \multicolumn{4}{c}{\textbf{Tiny-ImageNet}} & \multicolumn{4}{c}{\textbf{ImageNet-Subset}} \\
& \multicolumn{2}{c}{Split-10} & \multicolumn{2}{c}{Split-20} & \multicolumn{2}{c}{Split-10} & \multicolumn{2}{c}{Split-20} & \multicolumn{2}{c}{Split-10} & \multicolumn{2}{c}{Split-20} \\
\cmidrule(lr){2-3} \cmidrule(lr){4-5} \cmidrule(lr){6-7} \cmidrule(lr){8-9} \cmidrule(lr){10-11} \cmidrule(lr){12-13}
 & $A_{\text{inc}}$ & $A_{\text{last}}$ & $A_{\text{inc}}$ & $A_{\text{last}}$ & $A_{\text{inc}}$ & $A_{\text{last}}$ & $A_{\text{inc}}$ & $A_{\text{last}}$ & $A_{\text{inc}}$ & $A_{\text{last}}$ & $A_{\text{inc}}$ & $A_{\text{last}}$ \\
\midrule
EWC & 49.14 & 31.17 & 31.02 & 17.37 & 24.01 & 8.00 & 15.70 & 5.16 & 39.40 & 24.59 & 26.95 & 12.78 \\
LwF & 53.91 & 32.80 & 38.39 & 17.44 & 45.14 & 26.09 & 32.94 & 15.02 & 56.41 & 37.71  & 40.23 & 18.64 \\
PASS & 47.86 & 30.45 & 32.86 & 17.44 & 39.25 & 24.11 & 32.01 & 18.73 & 45.74 & 26.40 & 31.65 & 14.38 \\
FeTrIL & 51.20 & 34.94 & 38.48 & 23.28 & 45.60 & 30.97 & 39.54 & 25.70 & 52.63 & 36.17 & 42.43 & 26.63 \\
SSRE & 47.26 & 30.40 & 32.45 & 17.52 & 38.82 & 22.93 & 30.62 & 17.34 & 43.76 & 25.42 & 31.15 & 16.25 \\
EFC & 58.58 & 43.62 & 47.36 & 32.15 & \textit{47.95} & 34.10 & \textit{42.07} & \textit{28.69} & 59.94 & 47.38 & 49.92 & 35.75 \\
LDC & 59.50 & 45.40 & - & - & 46.80 & \textit{34.20} & - & - & \textit{69.40} & 51.40 & - & - \\
ADC & 61.35 & 46.48 & - & - & 43.04 & 32.32 & - & - & 67.07 & 46.58 & - & - \\
SEED & \textit{62.04} & \textit{51.42} & \textit{57.42} & \textit{42.87} & - & - & - & - & 67.55 & \textit{55.17} & \textit{62.26} & \textit{45.77} \\
\midrule
\multirow{2}{*}{\textbf{DCNet}}
 & \textbf{75.84} & \textbf{65.40} & \textbf{71.52} & \textbf{58.43} & \textbf{57.00} & \textbf{48.37} & \textbf{50.05} & \textbf{36.75} & \textbf{76.82} & \textbf{67.82} & \textbf{69.12} & \textbf{50.31} \\
 
 & ±0.52 & ±0.26 & ±0.42 & ±0.36 & ±0.22 & ±0.33 & ±0.10 & ±0.29 & ±0.25 & ±0.22 & ±0.43 & ±0.53 \\
\bottomrule
\end{tabular}
}
\end{center}
\vskip -0.1in
\caption{Comparison with baselines on Split CIFAR-100, Tiny-ImageNet, and ImageNet-Subset. All methods are trained from scratch without using replay samples. Our method is evaluated over five runs, with the mean performance and standard deviation reported. We emphasize the optimal results in bold and denote the sub-optimal results in italics.}

\vskip -0.2in
\label{Table_EF-CIL}
\end{table*}

\subsection{Dynamic Aggregation Compensation (DAC)}\label{DAC}
DAC focuses on the degree of intra-class aggregation. Specifically, as model plasticity decreases, the embeddings for subsequent tasks tend to become increasingly diffuse. To counteract this, DAC employs the aggregation patterns from previous tasks as a template, dynamically adjusting pressure to maintain consistent aggregation. We introduce and minimize an adaptive supervised contrastive loss to compensate for inadequate aggregation after the IOE has undergone a predefined number of training iterations:

\begin{align} \label{LossDAC}
  \mathcal{L}_{\text{DAC}} =
& -\frac{1}{N^{(t)}} \sum_{i=1}^{N^{(t)}} \frac{1}{|P(i)^{(t)}|} \times \notag \\
& \sum_{p \in P(i)^{(t)}} \log 
  \frac{\exp(\mathbf{z}_i^{(t)} \cdot \mathbf{z}_p^{(t)} / \tau^{(t)})}{\sum_{j=1, j \neq i}^N \exp(\mathbf{z}_i^{(t)} \cdot \mathbf{z}_j^{(t)} / \tau^{(t)})},
\end{align}

\noindent where $P(i)^{(t)}$ is the set of positive samples for sample $x_{i}$, $\mathbf{z}_i^{(t)}$ and $\mathbf{z}_p^{(t)}$ are the embedding representations of sample $x_{i}$ and its positive sample $x_{p}$, $\tau^{(t)}$ is an adaptive temperature that controls the compensation intensity. Adjusting temperature to optimize training for a single task is common; however, DAC aims to leverage information from incremental tasks to select the optimal $\tau^{(t)}$ for precise compensation. Indeed, employing fixed hyperparameters in CIL is suboptimal, as appropriate adjustments are necessary based on factors such as task complexity~\cite{semola2024adaptive,li2024harnessing}.

Naturally, we can estimate the concentration parameter $\kappa$ based on the distributional form of Eq. (\ref{vMF}), and subsequently employ this estimate as a degree of aggregation. However, due to the presence of the Bessel function, an analytic solution for $\kappa$ is not feasible. In DAC, we adopt a more intuitive approach by calculating the average cosine similarity between samples and basis vector to quantify the degree of aggregation $\omega^{(t)}$. This measure is then used to dynamically adjust the temperature relative to the historical average degree of aggregation $\omega^{\text {avg}}$:

\begin{equation} \label{temp}
\omega^{(t)} = \frac{1}{N^{(t)}} \sum_{i=1}^{N^{(t)}} \mathbf{z}_i^{(t)} \cdot \mu_{c_{(i)}}^{(t)}, \;\; \tau^{(t)} = \tau^{(0)} \cdot \frac{\omega^{(t)}}{\omega^{\text {avg}}}.
\end{equation}

DAC calculates the aggregation degree $\omega^{(t)}$ via Eq. (\ref{temp}) over a specified epoch and subsequently updates the temperature $\tau^{(t)}$ for Eq. (\ref{LossDAC}). If the aggregation degree of the current task is insufficient, DAC applies a lower temperature to enhance the embedding intensity; conversely, if the aggregation degree is adequate, the constraint is relaxed appropriately. This balancing mechanism uses the degree of intra-class aggregation as information to flow between tasks, thereby ensuring that each class maintains a consistent and compact embedding throughout the incremental learning process.

In summary, IOE and DAC enhance the inter-class separation and align the intra-class aggregation, constructing a discriminative and consistent feature space through the interaction between tasks. The final optimization objective comprises three components:

\begin{equation} \label{Loss}
\mathcal{L}_{\text{Total}} = \mathcal{L}_{\text{IOE}} + \lambda\cdot\mathcal{L}_{\text{DAC}} + \lambda_{\text{HAT}}\cdot\mathcal{L}_{\text{HAT}},
\end{equation}

\noindent where $\lambda$ and $\lambda_{\text{HAT}}$ are both hyper-parameters used to balance the total loss. The detailed algorithm for HAT and the procedure for DCNet are provided in Appendix~\ref{Algorithm}.

\section{Experiment}
\subsection{Experiment Setting}

\textbf{Datasets.} 
For a fair comparison with baselines, we utilize three widely adopted datasets in CIL. The CIFAR-100~\cite{CIFAR-100} comprises 50k training images and 10k test images, each sized 32×32 pixels, spanning 100 categories. The Tiny-ImageNet~\cite{Tiny-ImageNet-200}, a subset of ImageNet, includes 100k training images and 10k test images, each sized 64×64 pixels, covering 200 categories. The ImageNet-Subset is a subset of the ImageNet (ILSVRC 2012) ~\cite{russakovsky2015imagenet} with 100 categories, containing approximately 130k training images, each sized 224×224 pixels. We split these datasets equally into 10-task and 20-task sequences. This experimental setup is more challenging and realistic because it does not rely on a large initial task.

\textbf{Baselines.} 
Since our focus is on the EF-CIL scenario, we conduct comprehensive comparisons with both classical and state-of-the-art EF-CIL methods: \textbf{EWC}~\cite{PNAS2017EWC}, \textbf{LwF}~\cite{li2017learning}, \textbf{PASS}~\cite{zhu2021prototype}, \textbf{FeTrIL}~\cite{petit2023fetril}, \textbf{SSRE}~\cite{zhu2023SSRE}, \textbf{EFC}~\cite{magistri2024EFC}, \textbf{LDC}~\cite{gomez2025LDC}, \textbf{ADC}~\cite{goswami2024ADC}, \textbf{SEED}~\cite{rypesc2023SEED}. Furthermore, to substantiate the effectiveness of our approach, we also compare it with several state-of-the-art exemplar-based methods, particularly various TIL+OOD approaches, including: \textbf{iCaRL}~\cite{rebuffi2017icarl}, \textbf{DER++}~\cite{buzzega2020DER++}, \textbf{DER}~\cite{yan2021dynamically}, \textbf{FOSTER}~\cite{wang2022foster}, \textbf{BEEF}~\cite{wang2023beef}, \textbf{MORE}~\cite{kim2022MORE}, \textbf{ROW}~\cite{kim2023ROW}, \textbf{TPL}~\cite{lin2024TPL}.

\textbf{Training.} 
We employ a ResNet-18 model~\cite{he2016deep} trained from scratch for all experiments. For comparison baselines, we either reproduce the results using the hyperparameters specified in their source code repositories or directly adopt the existing results in state-of-the-art baselines. To ensure a fair comparison, we also allow the baselines to utilize self-rotation augmentation~\cite{magistri2024EFC}. For CIFAR-100 and Tiny-ImageNet, consistent with prior work~\cite{Kim2022CLOM}, we utilize LARS~\cite{you2017large} training for 700 epochs with an initial learning rate of 0.1, introducing the DAC component at epoch 400. For ImageNet-Subset, we train for 100 epochs, incorporating DAC at epoch 50. In all the experiments, we set the $\tau_{\text{IOE}} = 0.05$, $\tau^{(0)} = 0.2$ and configure the dimension of basis vector be 256.

\textbf{Evaluation.} 
We report two key metrics: $A_{\text{last}}$(\%), which represents the average accuracy after the last task; $A_{\text{inc}}$(\%), which denotes the average incremental accuracy across all tasks. For further details on training and metric calculations, please refer to Appendix~\ref{Training_Details}.

\subsection{Main Comparison Results}

\textbf{Comparison with EF-CIL Approaches.} 
Table \ref{Table_EF-CIL} offers a comprehensive comparison of various baselines across three standard benchmark datasets. All methods were trained from scratch without utilizing any replay samples. Our method demonstrates significant performance among all baselines with non-marginal improvements. Specifically, compared to the second-best method, DCNet achieves improvements of 11.01\% and 11.68\% in average and final accuracy metrics, respectively. This underscores the superior competitiveness of DCNet within EF-CIL. This outstanding performance can be attributed to our improvement of the TIL+OOD framework, which fundamentally differs from the traditional EF-CIL approach and effectively overcomes the isolation between tasks.

\begin{table}[ht]
\vskip -0.05in 
\begin{center}
	\begin{small}
        \scalebox{0.9}{
	\begin{tabular}{lccccccc}
        \toprule
        \textbf{Method} & \textbf{$\mathcal{M}$} & \textbf{CIFAR100} & \textbf{Tiny-ImageNet} & \textbf{ImageNet-Subset}\\
        \midrule
        iCaRL &  & 51.4 & 28.3 & 50.98 \\
        DER++ &  & 53.7 & 30.5 & - \\
        DER &  & \textit{64.5} & 38.3 & 66.85 \\
        FOSTER &  & 62.5 & 36.4 & 67.68 \\
        BEEF & 2k & 60.9 & 37.9 & \textbf{68.78} \\
        MORE$^\dagger$ &  & 57.5 & 35.4 & - \\
        ROW$^\dagger$ &  & 58.2 & 38.2 & - \\
        TPL$^\dagger$ &  & 62.2 & \textit{42.9} & - \\
         \midrule
        \textbf{DCNet} & 0 & \textbf{65.4} & \textbf{48.4} & \textit{67.82} \\
        \bottomrule
        \end{tabular}
		}
        \end{small}
\end{center}
\vskip -0.1in 
\caption{Comparison with exemplar-based baselines, where we report the average accuracy denoted as $A_{\text{last}}$. $^\dagger$: These methods also adopt the same TIL+OOD framework.}
\label{Table_Additional_Analysis}
\end{table}

\textbf{Comparison with Exemplar-based Approaches.}
Table \ref{Table_Additional_Analysis} presents a comparative analysis of our approach against several state-of-the-art exemplar-based methods. All baselines train from scratch while maintaining a buffer $\mathcal{M}$ of 2000 samples. DCNet achieves a performance improvement of 0.9\% and 5.5\% over these advanced methods on two 10 tasks sequences. In the ImageNet-Subset-split10 task, our method slightly underperforms the state-of-the-art baseline. It is important to note that maintaining a large buffer is an effective strategy for complex tasks. Specifically, buffering 297 samples from ImageNet consumes memory equivalent to that of a ResNet-18 backbone network. Therefore, a fair comparison should account for these resource requirements~\cite{zhou2022model}. Furthermore, we highlight that MORE, ROW, and TPL are also TIL+OOD methods that explicitly use replay samples to break task isolation. In contrast, our approach constructs the interaction of information among tasks without replay samples, and achieves performance that is competitive with exemplar-based methods.


\begin{table}[t]
\vskip -0.05in 

\begin{center}

\scalebox{0.9}{
\begin{tabular}{lcccccc}
\toprule
\textbf{Component} & \textbf{$\lambda$} &\multicolumn{2}{c}{\textbf{CIFAR100-10}} & \multicolumn{2}{c}{\textbf{ImageNet100-10}} \\
\cmidrule(lr){3-4} \cmidrule(lr){5-6}
& & $A_{\text{inc}}$ & $A_{\text{last}}$ & $A_{\text{inc}}$ & $A_{\text{last}}$ \\
\midrule
HAT+CSI &- & 73.30 & 63.32 & 70.80 & 63.94 \\
IOE      & -  & 73.85 & 63.80 & 73.55 & 64.86 \\
IOE+DAC$^\dagger$  & 1.0   & 74.04 & 64.40 & 75.09 & 65.80 \\
IOE+DAC$^{\dagger\dagger}$ & 0.5 & 75.49 & 65.27 & 76.19 & 66.68 \\
IOE+DAC$^{\dagger\dagger}$ & 1.0 & \textbf{75.84} & \textbf{65.40} & \textbf{76.82} & \textbf{67.82} \\
IOE+DAC$^{\dagger\dagger}$ & 2.0 & 75.58 & 64.92 & 75.64 & 65.92 \\
\bottomrule
\end{tabular}
}

\end{center}
\vskip -0.1in 
\caption{Effectiveness of the core designs in our DCNet. HAT+CSI serves as the foundation for our approach. $^\dagger$: DAC component with fixed temperature; $^{\dagger\dagger}$: DAC component with dynamic temperature.}
\label{tab:comparison}
\vskip -0.1in 
\end{table}


\begin{figure}[t]
\centering
\subfigure[]{
\includegraphics[width=1.1in]{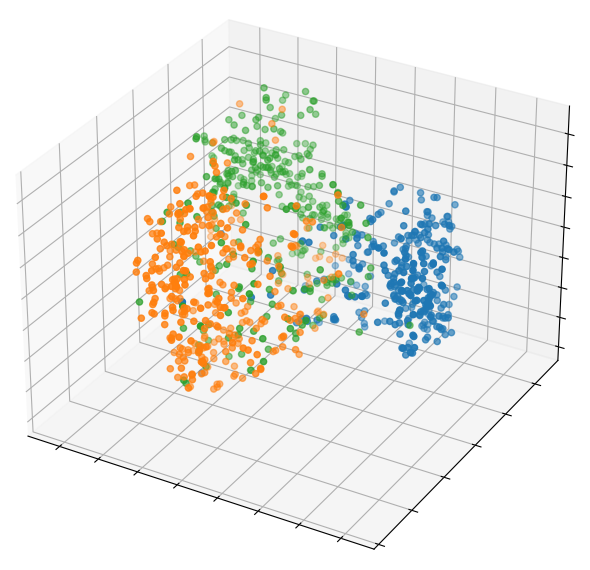}}
\subfigure[]{
\includegraphics[width=1.1in]{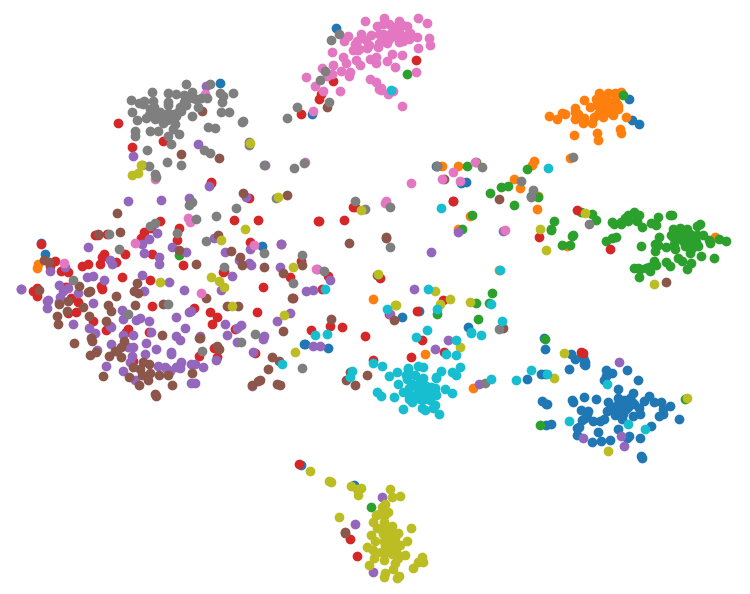}}

\vspace{-2mm}
\subfigure[]{
\includegraphics[width=1.1in]{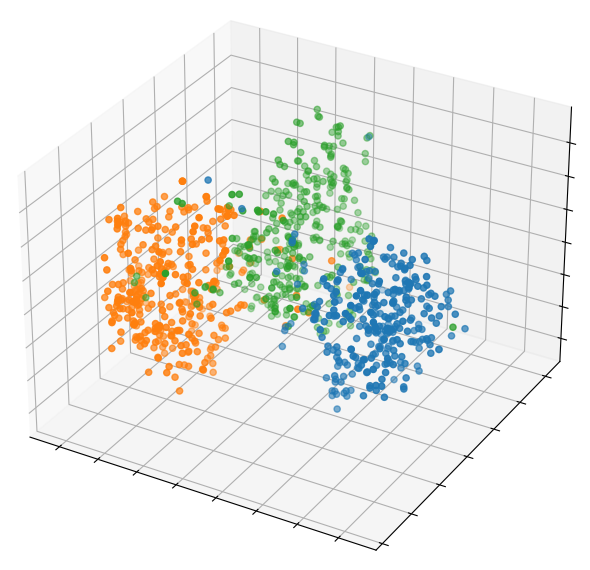}}
\subfigure[]{
\includegraphics[width=1.1in]{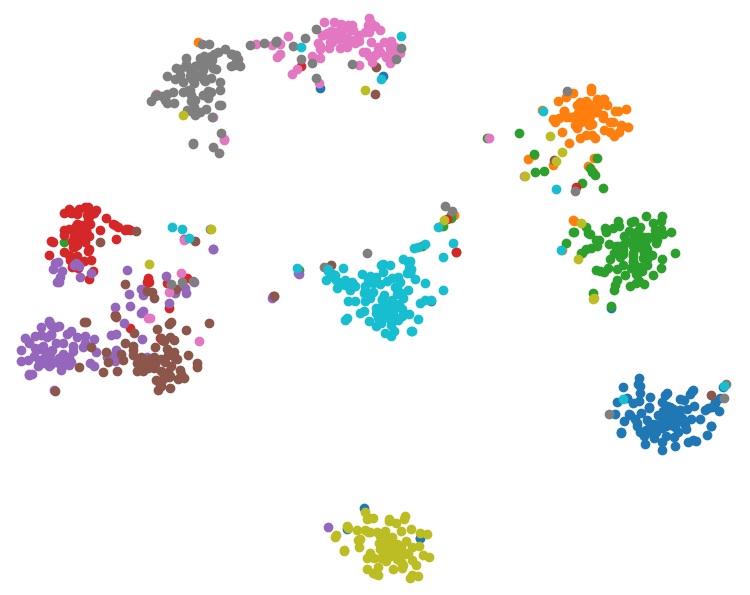}}

\vspace{-2mm}
\caption{t-SNE visualization of the embedding space, where a color represents a category. (a) (b) illustrate the embedding results for HAT+CSI; and (c) (d) present results for our method.} 
\label{figure1}
\end{figure}

\begin{figure}[t]
\vskip -0.02in 
\centering
\makeatletter\def\@captype{figure}\makeatother
\begin{minipage}{0.22\textwidth}
	\begin{center}
		\centerline{\includegraphics[width=\columnwidth, trim=90 50 120 110, clip]{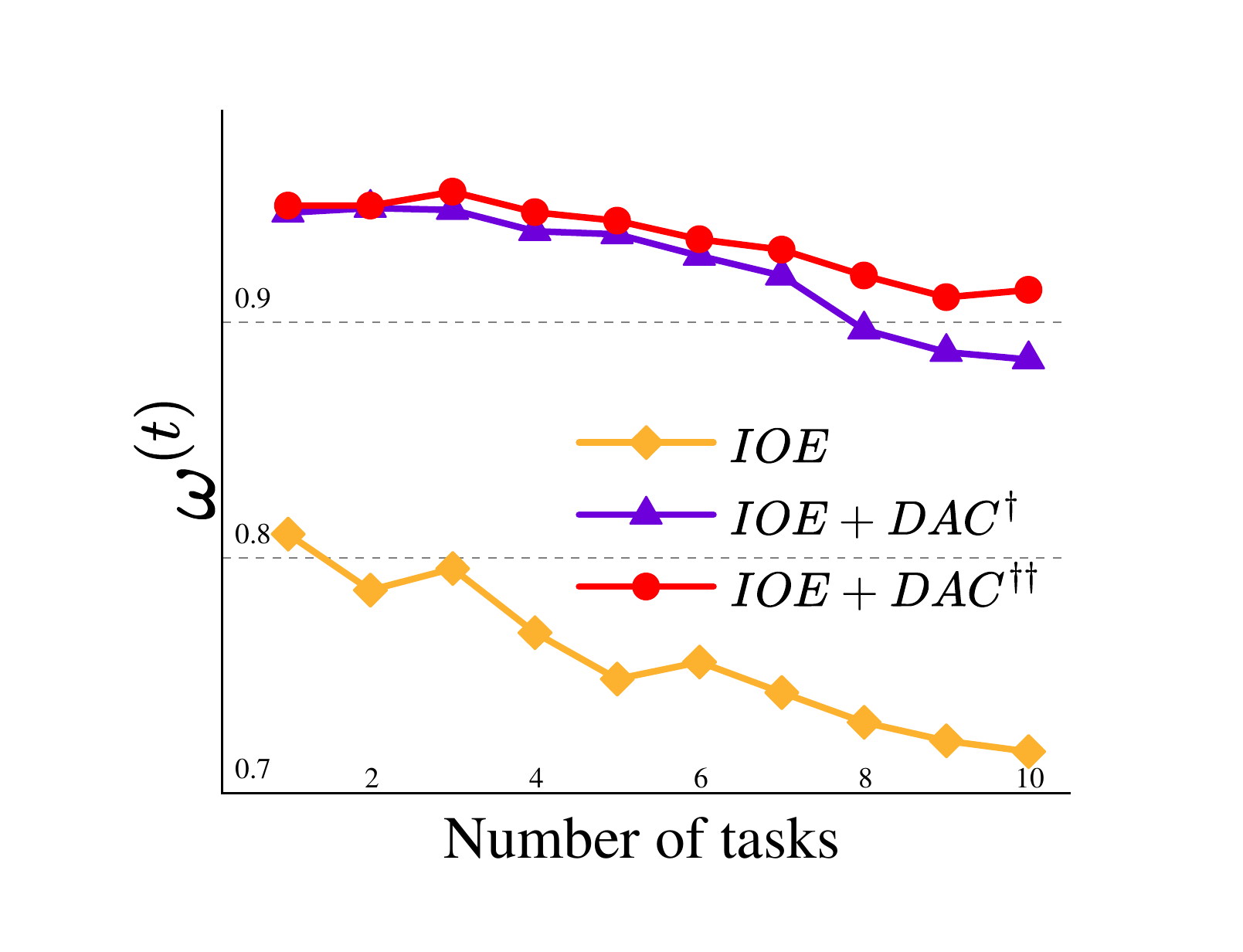}}
		\caption{Changes in the degree of aggregation $\omega^{(t)}$.}
		\label{figure2}
	\end{center}
	\vskip -0.3in
\vskip 0.1in
\end{minipage}
\hspace{0.1in} 
\centering
\makeatletter\def\@captype{figure}\makeatother
\begin{minipage}{0.23\textwidth}
	\begin{center}
		\centerline{\includegraphics[width=\columnwidth, trim=90 50 50 70, clip]{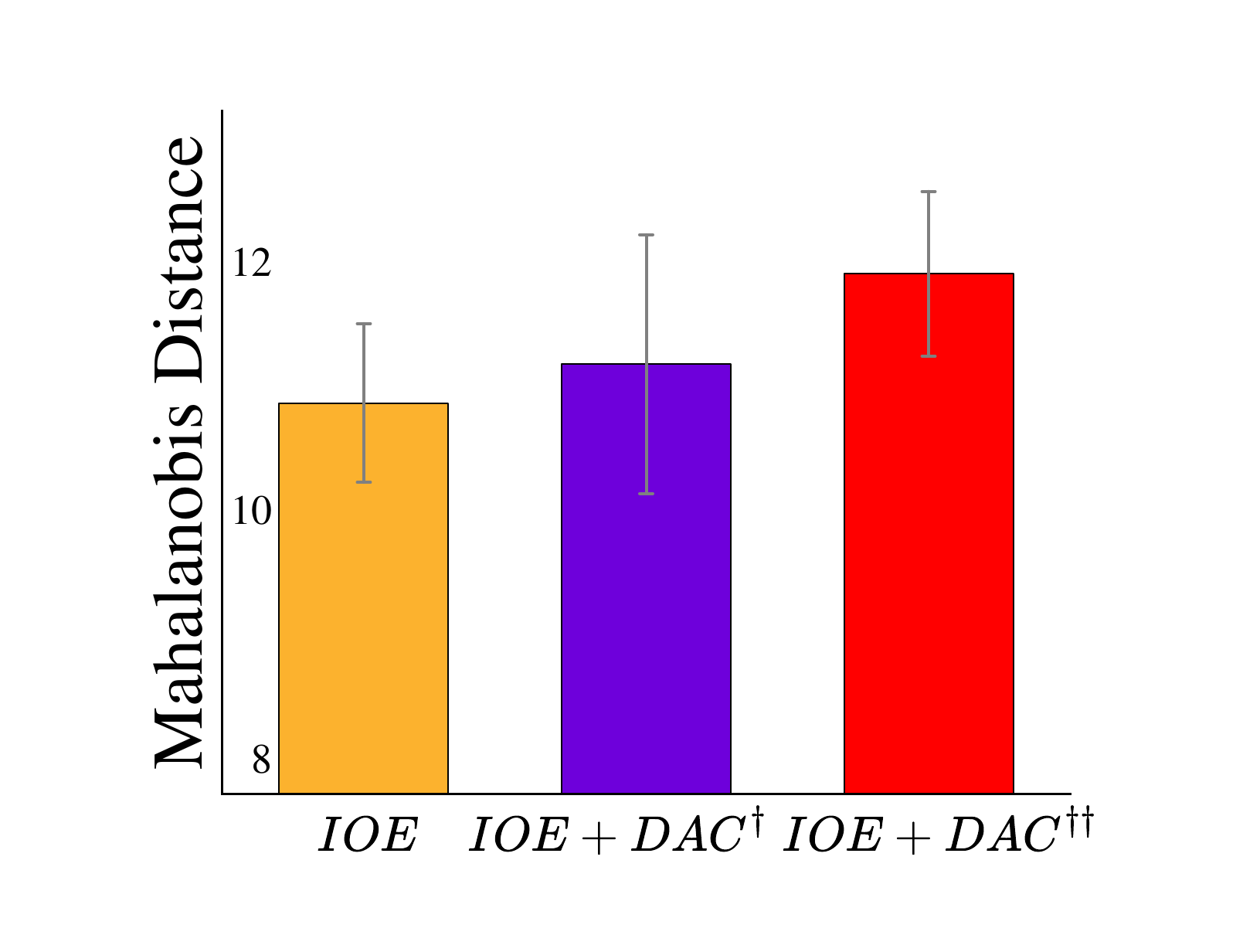}}
        
		\caption{Average Mahalanobis distance between classes.}
		\label{figure3}
	\end{center}
	\vskip -0.3in
\vskip 0.1in
\end{minipage}
\end{figure}

\subsection{Algorithm Analysis} \label{Ablation Study}
We evaluate the empirical effectiveness of DCNet. Table~\ref{tab:comparison} presents the results of the ablation study and parametric analysis. HAT+CSI, a pioneering approach for TIL+OOD, is introduced by Kim et al.~\shortcite{kim2022theoretical} to demonstrate the feasibility of the TIL+OOD framework; however, it does not consider the information interaction between tasks. Consequently, HAT+CSI can be considered a precursor to DCNet.

Building on this, Table~\ref{tab:comparison} provides results for three configurations: IOE, using only the IOE component; IOE+DAC$^\dagger$, combining the fixed DAC component with IOE; IOE+DAC$^{\dagger\dagger}$, combining the dynamic DAC component with IOE using different $\lambda$ for Eqs. (\ref{Loss}). The results indicate that our methods are interdependent and achieve superior performance by fully utilizing available information. It is noteworthy that the performance gains of DCNet are particularly pronounced in the more complex ImageNet-Subset task.

Figure~\ref{figure1} visualizes the feature space of a single task, where (a) (b) present the results of HAT+CSI and (c) (d) represent our method. Each group visualizes three distinct classes in a three-dimensional space and an entire task in a two-dimensional space. Attribute to the superior inter-class separation by IOE, the feature space of DCNet is more discriminative. Figure~\ref{figure2} illustrates the decreasing trend of intra-class aggregation $\omega^{(t)}$ as tasks progress, due to diminished model plasticity. To counteract this effect, the DAC component dynamically adjusts the compensation intensity based on the degree of aggregation from previous tasks, thereby stabilizing the change curve of intra-class aggregation. Finally, Figure~\ref{figure3} presents the average inter-class Mahalanobis distance across the incremental sequence. By leveraging the two components, DCNet constructs a feature space that is both discriminative and consistent.

\section{Conclusion}
We introduce a novel approach of TIL+OOD framework in the context of EF-CIL, leveraging information from incremental sequence to overcome task isolation. Theoretical analysis reveals that inter-class separation and intra-class aggregation are crucial for effective OOD detection in an incremental learning sequence. Our proposed DCNet, informed by these insights, preserves the discriminability and consistency of the feature space via its IOE and DAC components. Extensive experiments validate the competitiveness of our method. Future research could explore scenarios with blurry task boundaries, where models need to be aware of new class arrivals during training.
\newpage
\bibliographystyle{named}
\bibliography{ijcai25}

\newpage
\appendix
\onecolumn
%
%

\section{Proof of Theorem \ref{thm:extended}} \label{Proof_Theorem_2}

Recalling Eqs.~\ref{ES(x)} and \ref{D}, our objective is to analyze the factors influencing $D$, which quantifies the performance of OOD detection under incremental sequences. First, we prove the following lemma, which elucidates the factors contributing to the lower bound.

\begin{lemma}
\label{lower bound}
Let
\(\alpha_{i,t} := 
\tfrac{1}{2}\, d_M\!\bigl(\mu_{\mathrm{in},i}, \mu_{\mathrm{out},t}\bigr),\; i = 1,\ldots, k,\; t = 1,\ldots, T\), then we have the following estimate,
\[
\mathbb{E}_{x \sim P^{\mathrm{out}}_\chi}(ES(x))
\leq \frac{1}{T} \sum_{t=1}^T \sum_{i=1}^k(
(1 - P^{\mathrm{out}}_\chi(B_{\alpha_{i,t}}(\mu_{\mathrm{out},t}))) + \exp\left(-\frac{\alpha_{i,t}^2}{2}\right)),
\]
$B_{r}(u)$ denotes the open ball of radius $r$ in the Mahalanobis distance, centered at $u$. $\mu_{\mathrm{in},i}$, $\mu_{\mathrm{out},t}$ can have arbitrary configurations.


\end{lemma}

\begin{proof} 
We have,
\[
\mathbb{E}_{x \sim P^{\mathrm{out}}_\chi}(ES_i(x)) =
\int_{\mathbb{R}^d} ES_i(x)p^{\mathrm{out}}_\chi(x) dx =
\frac{1}{T} \sum_{t=1}^T \int_{B_{\alpha_{i,t}}(\mu_{\mathrm{out},t})} ES_i(x)p^{\mathrm{out}}_\chi(x) dx + 
\frac{1}{T} \sum_{t=1}^T \int_{B_{\alpha_{i,t}}^c(\mu_{\mathrm{out},t})} ES_i(x)p^{\mathrm{out}}_\chi(x) dx.
\]
Next, for $x \in B_{\alpha_{i,t}}(\mu_{\mathrm{out},t})$, by the triangle inequality and the given definition, we have,
\[
d_M(\mu_{\mathrm{out},t}, x) + d_M(x, \mu_{\mathrm{in},i}) \geq d_M(\mu_{\mathrm{out},t}, \mu_{\mathrm{in},i}),
\]
\[
d_M(x, \mu_{\mathrm{in},i}) \geq \frac{1}{2}d_M(\mu_{\mathrm{in},i}, \mu_{\mathrm{out},t}).
\]
So for the first term, we have,
\[
\frac{1}{T} \sum_{t=1}^T \int_{B_{\alpha_{i,t}}(\mu_{\mathrm{out},t})} ES_i(x)p^{\mathrm{out}}_\chi(x) dx  \leq \frac{1}{T} \sum_{t=1}^T \exp\left(-\frac{1}{2}\alpha_{i,t}\right)P^{\mathrm{out}}_\chi(B_{\alpha_{i,t}}(\mu_{\mathrm{out},t})).
\]
For the second term, since $ES_i(x) \leq 1$, we have,
\[
\frac{1}{T} \sum_{t=1}^T \int_{B_{\alpha_{i,t}}^c(\mu_{\mathrm{out},t})} ES_i(x)p^{\mathrm{out}}_\chi(x) dx \leq \frac{1}{T} \sum_{t=1}^T (1 - P^{\mathrm{out}}_\chi(B_{\alpha_{i,t}}(\mu_{\mathrm{out},t}))).
\]
Putting all together, we have,
\[
\mathbb{E}_{x \sim P^{\mathrm{out}}_\chi}ES_i(x) \leq \frac{1}{T} \sum_{t=1}^T(
(1 - P^{\mathrm{out}}_\chi(B_{\alpha_{i,t}}(\mu_{\mathrm{out},t}))) + \exp\left(-\frac{\alpha_{i,t}^2}{2}\right)P^{\mathrm{out}}_\chi(B_{\alpha_{i,t}}(\mu_{\mathrm{out},t}))).
\]
Considering different $\mu_{\mathrm{in},i}$, we have,
\[
\mathbb{E}_{x \sim P^{\mathrm{out}}_\chi}(ES(x)) \leq \frac{1}{T} \sum_{t=1}^T \sum_{i=1}^k(
(1 - P^{\mathrm{out}}_\chi(B_{\alpha_{i,t}}(\mu_{\mathrm{out},t}))) + \exp\left(-\frac{\alpha_{i,t}^2}{2}\right)P^{\mathrm{out}}_\chi(B_{\alpha_{i,t}}(\mu_{\mathrm{out},t}))).
\]
As $P^{\mathrm{out}}_\chi(B_{\alpha_{i,t}}(\mu_{\mathrm{out},t})) \leq 1$, we have,
\[
\mathbb{E}_{x \sim P^{\mathrm{out}}_\chi}(ES(x)) \leq \frac{1}{T} \sum_{t=1}^T \sum_{i=1}^k(
(1 - P^{\mathrm{out}}_\chi(B_{\alpha_{i,t}}(\mu_{\mathrm{out},t}))) + \exp\left(-\frac{\alpha_{i,t}^2}{2}\right)).
\]
This completes the proof.
\end{proof}

Before analyzing the upper bound of $D$, we first recall the Total Variation.

\begin{lemma}\label{Total Variation} (Total Variation.) Let $P_1, P_2 \in P(\mathcal{X})$. The Total Variation is defined as:
\[
\delta(P_1, P_2) = \sup_{A \in \mathcal{B}} |P_1(A) - P_2(A)|,
\]
where $\mathcal{B}$ denotes the Borel $\sigma$-algebra on $\mathcal{X}$. Furthermore, let $\mathcal{F}$ denote the unit ball in $L^\infty(\mathcal{X})$,
\[
\mathcal{F} := \{f \in L^\infty(\mathcal{X}) \mid \|f\|_\infty \leq 1\}.
\]
Then, the total variation distance can be equivalently characterized as:
\[
\delta(P_1, P_2) = \sup_{f \in \mathcal{F}} \left| \mathbb{E}_{x \sim P_1} f(x) - \mathbb{E}_{x \sim P_2} f(x) \right|.
\]
\end{lemma}

Immediately following this, we recall that the Kullback-Leibler (KL) divergence for the multivariate Gaussian distribution.

\begin{lemma}\label{KLD} (Kullback-Leibler Divergence.) Let $P_1 \sim \mathcal{N}(\mu_1, \Sigma)$ and $P_2 \sim \mathcal{N}(\mu_2, \Sigma)$, then we have the following:
\[
KL(P_1 \| P_2) = \frac{1}{2} \left((\mu_1 - \mu_2)^\top \Sigma^{-1} (\mu_1 - \mu_2)\right) = \frac{1}{2} d_M^2(\mu_1, \mu_2).
\]
\end{lemma}

Next, we recall the following inequality that bounds the total variation by KL-divergence.

\begin{lemma}\label{Pinsker's inequality} (Pinsker's Inequality.) Let $P_1, P_2 \in P(\mathcal{X})$, then we have the following:
\[
\delta(P_1, P_2) \leq \sqrt{\frac{1}{2} KL(P_1 \| P_2)}.
\]
Furthermore, the subsequent version, which holds true when KL is large, is also valid:
\[
\delta(P_1, P_2) \leq 1 - \frac{1}{2} \exp\left(-KL(P_1 \| P_2)\right).
\]
\end{lemma}

Finally, we prove the upper bound of $D$.
\\
\\
\textbf{Proof of Lemma ~\ref{upper bound}}\label{proof2}
\begin{proof} First, notice that, for $1 \leq i \leq k$,
\[
ES_i(x) \in [0, 1] \implies ES(x) \in [0, k].
\]
Therefore, by Lemma~\ref{Total Variation}, we have,
\[
\mathbb{E}_{x \sim P_\chi^{\mathrm{in}}}(ES(x)) - \mathbb{E}_{x \sim P_\chi^{\mathrm{out}}}(ES(x)) \leq k \cdot \delta(P_\chi^{\mathrm{in}}, P_\chi^{\mathrm{out}}).
\]
Next recall,
\[
p_\chi^{\mathrm{in}}(x) = \frac{1}{k} \sum_{i=1}^k p_i^{\mathrm{in}}(x),\quad
p_\chi^{\mathrm{out}}(x) = \frac{1}{T} \sum_{t=1}^T p_t^{\mathrm{out}}(x).
\]
Therefore, let $P_i$ denote the probability distribution associated with $p_i$. By invoking the triangle inequality and the total variation, we have
\[
\delta\left(P_\chi^{\mathrm{in}}, P_\chi^{\mathrm{out}}\right) = \delta\left(\frac{1}{k} \sum_{i=1}^k P_i^{\mathrm{in}}, \frac{1}{T} \sum_{t=1}^T P_t^{\mathrm{out}}\right) = \sup_{A \subseteq \mathcal{B}} \left| \frac{1}{k} \sum_{i=1}^k P_i^{\mathrm{in}}(A) - \frac{1}{T} \sum_{t=1}^T P_t^{\mathrm{out}}(A) \right|
\]

\[
\leq \frac{1}{T} \frac{1}{k} \sum_{t=1}^T \sum_{i=1}^k \sup_{A \subseteq \mathcal{B}} |P_i^{\mathrm{in}}(A) - P_t^{\mathrm{out}}(A)|
\]

\[
= \frac{1}{T} \frac{1}{k} \sum_{t=1}^T \sum_{i=1}^k \delta(P_i^{\mathrm{in}}, P_t^{\mathrm{out}}).
\]
By Lemma~\ref{KLD} and ~\ref{Pinsker's inequality}, we have,
\[
\delta(P_i^{\mathrm{in}}, P_t^{\mathrm{out}}) \leq \sqrt{\frac{1}{2} KL(P_i^{\mathrm{in}} \| P_t^{\mathrm{out}})} = \frac{1}{2} \sqrt{(\mu_{\mathrm{in},i} - \mu_{\mathrm{out},t})^\top \Sigma^{-1} (\mu_{\mathrm{in},i} - \mu_{\mathrm{out},t})} = \frac{1}{2} d_M(\mu_{\mathrm{in},i}, \mu_{\mathrm{out},t}).
\]
Putting all together, we have
\[
\mathbb{E}_{x \sim P_\chi^{\mathrm{in}}}(ES(x)) - \mathbb{E}_{x \sim P_\chi^{\mathrm{out}}}(ES(x)) \leq \frac{1}{2T} \sum_{t=1}^T \sum_{i=1}^k d_M(\mu_{\mathrm{in},i}, \mu_{\mathrm{out},t}) = \frac{1}{T} \sum_{t=1}^T \sum_{i=1}^k \alpha_{i,t}.
\]
This completes the proof.
\end{proof}

Up to this point, we have established theoretical upper and lower bounds on the performance of OOD detection in incremental tasks. However, within the context of EF-CIL, $\mu_{\mathrm{out},t}$ is not estimable. Consequently, we now introduce the concept of inter-class separation and present the theorem in the main text.
\\
\\
\textbf{Proof of Theorem ~\ref{thm:extended}}\label{proof3}
\begin{proof}
By Lemma~\ref{upper bound}, we obtain an upper bound for $D$ as:
\[
\mathbb{E}_{x \sim P_\chi^{\mathrm{in}}}(ES(x)) - \mathbb{E}_{x \sim P_\chi^{\mathrm{out}}}(ES(x)) \leq \frac{1}{2T} \sum_{t=1}^T \sum_{i=1}^k d_M(\mu_{\mathrm{in},i}, \mu_{\mathrm{out},t}).
\]
Let
\[
i_{0}=\arg \min _{i=1, \ldots, k} d_{M}\left(\mu_{\mathrm{out},t}, \mu_{\mathrm{in}, i}\right),
\]
where $\mu_{\mathrm{in},i_0}$ is the IND mean closest to $\mu_{\mathrm{out}}$. By the triangle inequality, we have:
\[
d_M(\mu_{\mathrm{in},i}, \mu_{\mathrm{out},t}) \leq d_M(\mu_{\mathrm{out},t}, \mu_{\mathrm{in},i_0}) + d_M(\mu_{\mathrm{in},i_0}, \mu_{\mathrm{in},i}).
\]
Summing over all terms, we obtain:
\[
\sum_{t=1}^T \sum_{i=1}^k d_M(\mu_{\mathrm{in},i}, \mu_{\mathrm{out},t}) \leq \sum_{t=1}^T \sum_{i=1}^k \left(d_M(\mu_{\mathrm{out},t}, \mu_{\mathrm{in},i_0}) + d_M(\mu_{\mathrm{in},i_0}, \mu_{\mathrm{in},i})\right).
\]
Factoring out the constant terms, we have:
\[
  \mathbb{E}_{x \sim P_\chi^{\mathrm{in}}}\bigl[ES(x)\bigr]-
  \mathbb{E}_{x \sim P_\chi^{\mathrm{out}}}\bigl[ES(x)\bigr]
  \le
    \frac{k}{2T} \sum_{t=1}^T d_M\!\bigl(\mu_{\mathrm{out},t}, \mu_{\mathrm{in}, i_{0}}\bigr) +
    \frac{1}{2} \sum_{i=1}^k d_M\!\bigl(\mu_{\mathrm{in}, i_{0}}, \mu_{\mathrm{in},i}\bigr).
\]
This completes the proof.
\end{proof}

\section{Algorithm description for DCNet}
\label{Algorithm}
\subsection{Details of DCNet} \label{DCNet_Details}
To better illustrate our method, the whole procedure of training and testing is provided in Algorithm~\ref{Alg_DCNet}.
\begin{algorithm}[htbp]
	\caption{DCNet Training and Test Algorithm}
	\label{Alg_DCNet}
	\textbf{Input}: Datasets $\{\mathcal{D}^{(t)}\}_{t=1}^T$; Encoder $f$;  Initial temperature $\tau^{(0)}$; Epoch for IOE $epoch_{IOE}$; Epoch for DAC $epoch_{DAC}$, etc.

	\begin{algorithmic}[1] 
		\STATE \textbf{\textit{\# Training Time}} \\
		\FOR{$t=1, 2,\dots,T$}
        \STATE Generate orthogonal unit basis vectors $\mathbf{\mu}^{(t)}$ by Eq. (\ref{ebv})
        
        \FOR{$i=1, 2,\dots,epoch_{IOE}$}
        \STATE Embed the new categories into the neighborhood of $\mathbf{\mu}^{(t)}$ by Eq. (\ref{LossIOE})
		\STATE Calculate $L_{HAT}$ by comparing with $mask^{(t-1)}$
		\STATE Update encoder $f$ by Eq. (\ref{Loss}) without $\mathcal{L}_{DAC}$
		\ENDFOR
		\STATE Calculates the aggregation degree $\omega^{(t)}$ and
        update temperature $\tau^{(t)}$ by Eq. (\ref{temp})
        \FOR{$j=1, 2,\dots,epoch_{DAC}$}
		\STATE Compensate dynamically for the degree of aggregation by Eq. (\ref{LossDAC})
		\STATE Update encoder $f$ by Eq. (\ref{Loss})
		\ENDFOR
        \STATE Update $mask^{(t)}$, $\omega^{(avg)}$
        \STATE Create and train a new OOD classifier
        
		\ENDFOR \\
		
		\STATE \textbf{\textit{\# Inference Time}} \\
		\STATE Make a decision using the encoder $f$, $mask^{(T)}$ and all the OOD classifiers by Eq. (\ref{decision}) \\

	\end{algorithmic}
	\textbf{Output}: Task-id $t$ and the predicted classes $c$
\end{algorithm}

\subsection{Details of HAT} \label{HAT_Details}
We emphasize that various TIL techniques are viable within the TIL+OOD framework. However, to maintain consistency with prior research, we adopt hard attention mask (HAT)~\cite{serra2018HAT} to mitigate catastrophic forgetting.
At each layer $l$ of the network, there is a trainable binary vector $\alpha^t_l$, whose 
$i$-th element $a^t_{i,l}\in\{0,1\}$ selectively control the activation of neuron $i$. 
Concretely, the output of layer $l$, denoted $h_l$, is elementwise multiplied by 
$\alpha^t_l$, as:
\[
h^{'}_l = \alpha^t_l \odot h_l.
\]

A neuron $i$ is considered essential for task $t$ if $\alpha^t_{i,l} = 1$ (and thus is protected from modification), whereas neurons with $\alpha^t_{i,l} = 0$ are treated as unimportant for this task and can be freely updated. Before training task $t$, we need to update all accumulated masks $\alpha^{<t}_{l}$:
\[
\alpha^{<t}_{l} = \max(\alpha^{<t-1}_{l},\alpha^{t-1}_{l}).
\]

When training task $t$, we modify the gradients of parameters that were deemed important for previous tasks $1,\dots,t-1$, thereby minimizing interference of the new task. Denote by $w_{ij,l}$ the parameter in row $i$, column $j$
of layer $l$. Its gradient is updated as follows:
\[
\nabla w_{ij,l}^{'} \;\leftarrow\;
\Bigl(1 - \min\bigl(\alpha^{<t}_{i,l},\,\alpha^{<t}_{j,l-1}\bigr)\Bigr) 
\,\nabla w_{ij,l},
\]

\noindent where $\alpha^{<t}_{i,l}$ is the accumulated hard attention for neuron $i$ at layer $l$ over all previous tasks. In particular, $\alpha^{<t}_{i,l} = 1$ if that neuron was used by any of the prior tasks. Therefore, the model freezes the gradient of parameters used in prior tasks, thus reducing interference.

Given the limited total capacity of the network, we aim to promote parameter sharing, which in turn facilitates the transfer of knowledge from prior tasks to new ones. For this reason, we make the mask as sparse as possible and introduce a regularization term accordingly:
\[
\mathcal{L}_{HAT} \;=\; 
\frac{\sum_{l,i} \,\alpha^t_{i,l}\,\bigl(1-\alpha^{<t}_{i,l}\bigr)}
     {\sum_{l,i} \,\bigl(1-\alpha^{<t}_{i,l}\bigr)}.
\]

In essence, this term encourages the model to avoid trivially setting all mask elements to 1.

\section{Additional implementation details} \label{Training_Details}
\subsection{Hyper-parameter} 
\label{Hyper-parameter}
For all baselines, we utilize either the results reported in the latest state-of-the-art baseline or reproduce the experiments using the hyperparameters provided in the original source code to ensure a fair comparison. All methods employed a ResNet-18 network trained from scratch as the backbone, without leveraging any pre-trained models. For the examplar-based approach, consistent with prior work~\cite{lin2024TPL}, we configured the replay buffer size to 2000.

For our approach, we follow the setup of the previous work\cite{Kim2022CLOM,kim2022theoretical}. For the CIFAR-100 and Tiny-ImageNet datasets, we train the backbone for 700 epochs using LARS~\cite{you2017large} with an initial learning rate of 0.1 and a batch size of 64. The DAC component is introduced at the 400th epoch, and the OOD classifier is subsequently optimized for 100 epochs using SGD. For the ImageNet-Subset dataset, the backbone is trained for 100 epochs, with the DAC component introduced at the 50th epoch, followed by optimizing the OOD classifier for 50 epochs using SGD, and the batch size is 32. Consistent with prior researches\cite{Kim2022CLOM,kim2022theoretical}, we employ a cosine scheduler and self-rotation augmentation during training. For IOE, the feature dimension of the penultimate layer is set to 512, the basis vector dimension to 256, and a cosine value of 0.1 is used to approximate orthogonality between vectors. The temperature $\tau_{\text{IOE}}$ is set to 0.05. For DAC, the initial temperature $\tau^{(0)}$ is set to 0.2 and dynamically adjusted based on the task. Regarding loss weight, $\mathcal{L}_{DAC}$ is assigned a weight of 1. For $\mathcal{L}_{HAT}$, the weight follows the prior setting\cite{Kim2022CLOM}; for instance, in the CIFAR100-10 task, it is 1.5 for the first task and 1 for subsequent tasks.
All experiments are implemented in PyTorch using NVIDIA RTX 3080-Ti GPUs.

\subsection{Metrics} 
\label{Metrics}
First we define the average accuracy after task $N$, where $R_{N,t}$ is the test accuracy for task $t$ after training on task $N$:
\[
A_{N}=\frac{1}{N}\sum_{t=1}^N R_{N,t}.
\]

When $N$ equal to the total number of tasks $T$, we get \textbf{$A_{\text{last}}$}, which represents the average accuracy after the last task; \textbf{$A_{\text{inc}}$} represents the average incremental accuracy across all tasks:
\[
A_{\text{last}}=\frac{1}{T}\sum_{t=1}^T R_{T,t},\quad
A_{\text{inc}}=\frac{1}{T}\sum_{N=1}^T A_{N}.
\]

\end{document}